\documentclass[a4paper,11pt]{article}

\pdfoutput=1

\usepackage[margin=1.02in]{geometry}
\usepackage[
  bookmarks=true,
  bookmarksnumbered=true,
  bookmarksopen=true,
  pdfborder={0 0 0},
  breaklinks=true,
  colorlinks=true,
  linkcolor=black,
  citecolor=black,
  filecolor=black,
  urlcolor=black,
]{hyperref}
\usepackage{isomath}
\usepackage[round]{natbib}
\usepackage{amsmath,amsfonts,amssymb,amsthm,enumitem}
\usepackage{mathtools}
\usepackage[capitalize]{cleveref}
\usepackage{nicefrac}
\usepackage{framed}
\usepackage{verbatim}

\theoremstyle{plain}
\newtheorem{theorem}{Theorem}[section]

\newtheorem{lemma}[theorem]{Lemma}
\newtheorem{corollary}[theorem]{Corollary}

\newtheorem{obs}[theorem]{Observation}
\newtheorem{claim}[theorem]{Claim}

\newtheorem{definition}[theorem]{Definition}

\DeclareMathOperator*{\E}{\mathbb{E}}
\newcommand{\Ex}{\E}

\providecommand{\F}{\mathcal{F}}

\providecommand{\eps}{\epsilon}

\newcommand{\ip}[1]{\langle #1 \rangle}
\newcommand{\R}{\mathbb{R}}
\newcommand{\N}{\mathbb{N}}

\newcommand{\Q}{\mathcal{Q}}
\newcommand{\C}{\mathcal{C}}

\newcommand{\sign}{\text{sign}}

\newcommand{\conf}{\text{Conf}}

\newcommand{\eqdef}{\triangleq}
\newcommand\underrel[2]{\mathrel{\mathop{#2}\limits_{#1}}}
\newcommand{\yields}[1]{\underrel{#1}{\implies}}

\newcommand{\vol}{\text{vol}}

\newcommand{\us}{ S}  
\newcommand{\ls}{ \bar S}      

\newif\ifdraft
\drafttrue

\ifdraft
\newcommand{\shay}[1]{{\color{red}{Shay:~#1}}}
\newcommand{\shachar}[1]{{\color{red}{Shachar:~#1}}}
\newcommand{\jiapeng}[1]{{\color{blue}{Jiapeng:~#1}}}

\else

\newcommand{\shay}[1]{}
\newcommand{\shachar}[1]{}
\newcommand{\jiapeng}[1]{}

\fi

\title{Active classification with comparison queries}
\author{
Daniel M. Kane\thanks{Department of Computer Science and Engineering/Department of Mathematics, University of California, San Diego. {\tt dakane@ucsd.edu} Supported by NSF Career Award ID 1553288.}
\and Shachar Lovett\thanks{Department of Computer Science and Engineering, University of California, San Diego. {\tt slovett@cs.ucsd.edu.} Research supported by NSF CAREER award 1350481.}
\and Shay Moran\thanks{Department of Computer Science and Engineering, University of California, San Diego, 
Simons Institute for the Theory of Computing, Berkeley, and Max Planck Institute for Informatics, Saarbr\"{u}cken, Germany. {\tt  shaymoran1@gmail.com.}}
\and  Jiapeng Zhang\thanks{Department of Computer Science and Engineering, University of California, San Diego. {\tt jpeng.zhang@gmail.com.} Research supported by NSF CAREER award 1350481.}
}

\begin{document}

\maketitle

\thispagestyle{empty}
\begin{abstract}

We study an extension of active learning in which the learning algorithm may ask the annotator to compare the distances of two examples
from the boundary of their label-class. For example, in a recommendation system application (say for restaurants),
the annotator may be asked whether she liked or disliked a specific restaurant (a label query); or which one of two restaurants did she like more (a comparison query).

We focus on the class of half spaces, and show that under natural assumptions, such as large margin or bounded bit-description of the input examples,
it is possible to reveal all the labels of a sample of size $n$ using approximately $O(\log n)$ queries. This implies an exponential improvement
over classical active learning, where only label queries are allowed. We complement these results by showing that if any of these assumptions is removed
then, in the worst case, $\Omega(n)$ queries are required.

Our results follow from a new general framework of active learning with additional queries. We identify a combinatorial dimension, called
the \emph{inference dimension}, that captures the query complexity when each additional query is determined by $O(1)$ examples (such as comparison queries, each of which is determined by the two compared examples).
Our results for half spaces follow by bounding the inference dimension in the cases discussed above.
\end{abstract}


\clearpage

\section{Introduction}
\label{sec:intro}
A central goal of {\em interactive learning} is understanding what
type of interaction between a learner and a domain expert enhances the learning
process, compared to the classical {\em passive learning} from labeled examples.

A basic model that was studied in this context is
\emph{pool-based active learning}~\citep{McCallumN98}.
Here, the algorithm has
an access to a large pool of {\em unlabeled examples} from which
it can pick examples and query their labels.
The goal is to make as few queries as possible while achieving
generalization-guarantees which are comparable with these of a passive algorithm
with an access to all of the labels.

A canonical example that demonstrates an advantage of active learning
is the class of threshold functions\footnote{
These are ``$\mathbb{R}\to\{\pm 1\}$'' functions
of the form $c(x)=\sign(a\cdot x-b)$, where $a,b\in\mathbb{R}$.}
over the real line.
Indeed, let $c$ denote the learned threshold function,
and let  $x_1<x_2<\ldots<x_n$ in $\mathbb{R}$
be the given pool of unlabeled examples.
It is possible to infer the labels of all $n$ points by making at most $\log n+2$ queries:
query the labels of the extreme points $c(x_1),c(x_n)$; if $c(x_1)=c(x_n)$
then the remaining points must be labeled the same;
otherwise, continue in a binary search fashion,
by labeling the middle point of the interval whose extreme points have opposite labels.
After at most $\log n$ such queries, the labels of all $n$ points are revealed.

Unfortunately, this exponential improvement in the query complexity
breaks for more general concept classes.
In fact, even for the class of 2 dimensional threshold functions\footnote{
These are ``$\mathbb{R}^2\to\{\pm1\}$'' functions of the form
$c(x)=\sign\bigl(\langle a,x\rangle-b\bigr)$,
where $a\in\mathbb{R}^2,b\in\mathbb{R}$.},
namely the class of half-planes, the (worst-case) query complexity
of active learning equals that of passive learning (see e.g.~\cite{Dasgupta04}).
Consequently, much of the literature was dedicated to developing
theory that takes into consideration further properties
of the unknown underlying distribution or the target concept~\citep{BalcanBZ07,Hanneke07,DasguptaHM08,BalcanBL09,BeygelzimerHLZ10,BalcanHV10,Koltchinskii10,BalcanHanneke11,HannekeYang12,El-YanivW12,BalcanLong13,GonenSS13,UrnerWB13,ZhangC14,
WienerHE15,BerlindU15}.

We consider another approach by
allowing the learning algorithm to further interact
with the domain expert by asking additional queries.
This poses a question:
\begin{center}
\emph{Which additional queries can the algorithm use?}
\end{center}
Allowing arbitrary queries will result in
a very strong learner (indeed, by halving the set of potential hypothesis
in every query, the number of queries can be made logarithmic).
However, arbitrary queries are useless in practice:
as experimental work by \cite{LangBaum92} shows,
algorithms that use set of queries that are too rich
may result in a poor practical performance.
This is not surprising if we keep in mind
that the human annotator who answers the queries
is restricted (computationally and in other ways).
Thus, a crucial factor in choosing the additional queries
is compatibility with the annotator who answers them.
A popular type of queries, which is used
in applications involving human annotators,
is \emph{relative queries}.
These queries poll relative information
between two or more data points.

This work focuses on
a basic kind of relative queries | \emph{comparison queries}.
Using such queries is sensible in settings in which
there is some natural ordering of the instances with respect to the learned concept.
As a toy example,
consider the goal of classifying films
according to whether a certain individual
is likely to enjoy them or not
(e.g.\ for recommending new films for this person).
In this context, the input sample consists
of films watched by the individual,
a label query asks whether
the person liked a film,
and a comparison query asks which of two
given films did the individual prefer.
As we will see, comparison queries
may significantly help.

Another aspect implied by the restricted nature of the human annotator
is that the learned concept resides in the class of concepts
that can be computed by the annotator, which presumably has low capacity.
Thus, realizability assumptions about the
generating data distributions are plausible in this context.

\subsection{Active learning with additional comparison queries}

Consider a learned concept of the form $c(x)=\sign\bigl(f(x)\bigr)$,
where $f$ is a real valued function (e.g.\ half spaces, neural nets),
and consider two instances $x_1,x_2$
such that, say $f(x_1)=10$, and $f(x_2)=1000$. Both $c(x_1)$ and $c(x_2)$ equal $+1$,
however that $f(x_2)>>f(x_1)$ suggests that $x_2$
is  a ``more positive instance'' than $x_2$.
In the setting of film classification
this is naturally interpreted
as that the person likes the film~$x_2$ more than the film~$x_1$.
We call the query ``$f(x_2) \ge f(x_1)$?''  \emph{a comparison query}.

\subsubsection{Example: learning half-planes with comparison and label queries}
\label{sec:example}

To be concrete, consider
the class of half-planes in $\mathbb{R}^2$. Here, a comparison query
is equivalent to asking which of two sample points $x_1,x_2$ lies closer to the boundary
line. Do such queries improve the query complexity over standard active learning?
It is known that without such queries, in the worst-case,
the learner essentially has to query all labels~\citep{Dasgupta04}.

The following algorithm demonstrates an exponential
improvement when comparison queries are allowed.
Later, we will present more general results showing
that this can be generalized to higher dimensions under
some natural restrictions, and that such restrictions
are indeed necessary.


\begin{framed}
\begin{center}
{\bf Interactive learning algorithm with comparison queries for half planes in $\R^2$}\\ {\footnotesize (see Figure~\ref{fig:plane} for a graphical illustration)}
\end{center}\label{alg1}
\noindent Setting: an unlabeled input sample of $n$ points $x_1,\ldots,x_n\in\R^2$ with hidden labels according to
a half-plane $c(x) = \sign\bigl(f(x)\bigr)$, where $f:\R^2\to \R$ is affine.
\\
\\
\noindent Repeat until all points are labeled:
\begin{enumerate}
\item Sample uniformly a subsample $S$ of $30$ points.
\item Query the labels of the points in $S$, and denote by:\\
$\mathcal{P}$ | points in $S$ labeled by $+1$,\\
$\mathcal{N}$ | points in $S$ labeled by $-1$.
\item Use comparison queries to find:\\
(i) $q$ -- the closest point in $\mathcal{P}$ to the boundary line,\\
(ii) $v$ -- the closest point in $\mathcal{N}$ to the boundary line.
\item Denote by:\\
 $[p,q],[q,r]$ | the two edges of the convex hull of $\mathcal{P}$ that are incident to $q$,\\
$[u,v],[v,w]$ |  the two edges of the convex hull of $\mathcal{N}$ that are incident to~$v$.
\item Infer that:\\
all points inside the cone $\angle pqr$ are labeled $+1$,\\
all points inside the cone $\angle uvw$ are labeled $-1$.
\item Repeat on the remaining unlabeled points.
\end{enumerate}
\end{framed}
\begin{figure}
\begin{center}
\includegraphics[width=.99\textwidth]{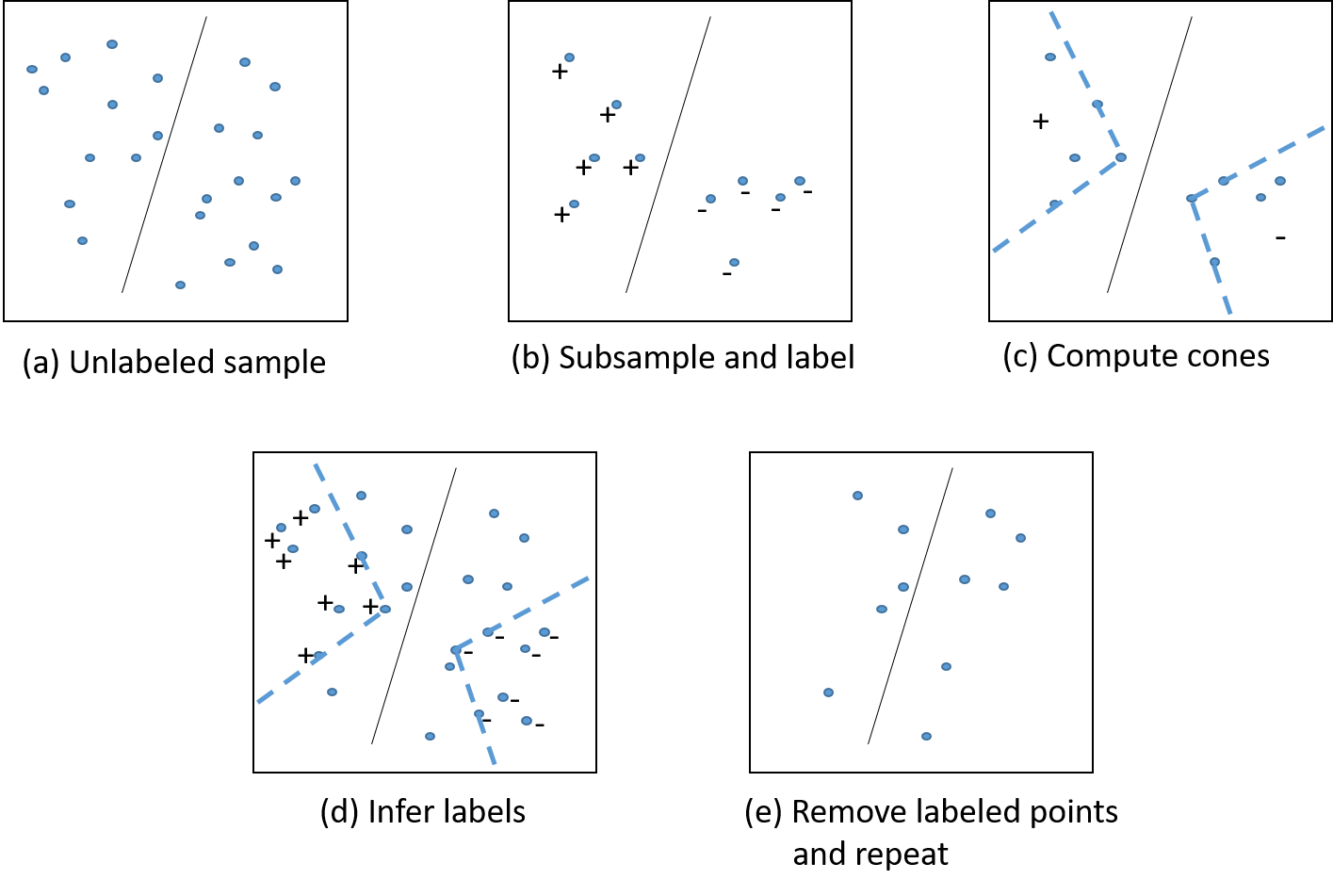}
\end{center}
\caption{An illustration of a single iteration.}
\label{fig:plane}
\end{figure}

The algorithm proceeds by iterations: it repeats steps 1-5 until all points are labeled.
At each iteration at most 60 queries are performed: 30 label queries in step 2 and at most 30 comparison queries in step 3,
for finding the points $q,v$ of minimal distance.
In each iteration, the algorithm infers the labels of all points
in the union of the angles $\angle pqr,\angle uvw$.
We will refer to this region as {\em ``the confident region''}.

We claim that after an expected number of $O(\log n)$ iterations
(and therefore using only $O(\log n)$ queries)
the algorithm infers the labels of all input points.
Establishing this statement boils down,
via a boosting argument (Theorem~\ref{thm:boosting})
to the following two important properties, which are easy to verify:
\begin{itemize}
\item {\bf Confidence:} every point which is labeled by the algorithm is labeled correctly.
\item {\bf Inference from a small subsample:} the confident region is determined by a subsample of 6 labeled points of~$S$ (i.e.\ $p,q,r$ and $u,v,w$).
\end{itemize}
These properties imply that at each iteration,
with probability at least $1/2$,
half of the remaining unlabeled points are labeled correctly (see Lemma~\ref{lem:weak}).
Thus, the expected number of queries is $O(\log n)$.

\paragraph{Paper organization.}
In Section~\ref{sec:results} we present and discuss our results,
later, in Section~\ref{sec:related} we survey previous related works,
and in Section~\ref{sec:future} we ceil the introduction
with suggested directions of future research and open problems.
Sections~\ref{sec:pre},~\ref{sec:infdim}, and~\ref{sec:comparison}
contain the technical definitions and proofs.

\subsection{Results}
\label{sec:results}
We next present and discuss our main results.
\begin{enumerate}
\item
Subsection~\ref{sec:inthalf} is dedicated to our
results concerning half spaces,
\item
Subsection~\ref{sec:intinf}
is dedicated to the inference dimension, 
and how it captures the query complexity
of active learning with additional comparison queries, and
\item
Subsection~\ref{sec:intgen} 
focuses on the general framework
of active learning with additional queries.
\end{enumerate}

\subsubsection{Interactive learning of half spaces with comparison queries}
\label{sec:inthalf}
We start by discussing our results for interactive learning of half spaces in $\R^d$ when both label queries
and comparison queries are allowed. We show that a general algorithm, as the one we described for $\R^2$,
cannot exist for $d \ge 3$. However, we identify two useful properties that allow for such a learning algorithm:
bounded bit complexity and margin.

\paragraph{Exact recovery of labels.}
We first describe our results in the context of exact recovery. Here,  the labels of all $n$ sample points needs to be revealed,
using as few queries as possible.

We first show that in the worst case, in $\R^d$ this requires $\Omega(n)$ queries for any $d \ge 3$ (recall that $O(\log n)$ queries suffice in~$\R^2$).
\begin{theorem}[Theorem~\ref{thm:lbdim}, informal version]
There are $n$ points in $\R^3$ that require
$\Omega(n)$ label and comparison queries for revealing all labels.
\end{theorem}

Our first positive result shows that efficient exact recovery of labels is possible if the points have low bit complexity.
\begin{theorem}[Theorem~\ref{thm:bit}, informal version]
Consider an arbitrary realizable sample of $n$ points in $\R^d$
whose individual bit complexity is $B$.
The labels of all sample points can be learned using $\tilde O(B\log n)$
label and comparison queries.
\end{theorem}

As an example, consider the sample consisting
of the $2^N$ point in the boolean hypercube $\{0,1\}^N$.
The bit complexity of every point is $N$.
The above thoerem implies that
given an unknown threshold function on $\{0,1\}^N$,
it is possible to reveal all $2^N$ labels using $\tilde O(N^2)$ comparison and label queries.
This should be compared to the situation where only label queries are allowed,
where all $2^N$ queries are necessary.
We discuss this example in more detail in Section~\ref{sec:thresholds},
and leave as an open question whether
revealing the $2^N$ labels can be done efficiently
(the above $\tilde O(N^2)$ bound applies only
to the query complexity, and not to the total running time).

Our second positive result shows that a similar algorithm exists if the margin is large.
\begin{theorem}[Theorem~\ref{thm:marg+dim}, informal version]
Assume a sample of $n$ points in $\R^d$ with maximal $\ell_2$ norm $\rho$
and margin at least $\gamma$. The labels of all points can be recovered using
$\tilde O\bigl(d\log(\nicefrac{\rho}{\gamma})\log n\bigr)$ label and comparison queries.
\end{theorem}
Our bound is in fact stronger: it is
\[\tilde O\bigl(d\log(\nicefrac{1}{\eta})\log n\bigr),\]
where
$\eta$ is the \emph{minimal-ratio} of the input sample,
defined by $\frac{\lvert \min_{i} f(x_i) \rvert}{\lvert\max_{i} f(x_i)\rvert}$,
where $x_1,\ldots, x_n$ are the sample points and $\sign\bigl(f(x)\bigr)$ is the learned concept.
In Section~\ref{sec:ub-margin} we show that the maximal ratio is lower bounded by the margin (and therefore yields a stronger statement).

Note that the above bound depends on $\nicefrac{\rho}{\gamma}$ logarithmically, and therefore applies
even in settings when the margin is exponentially small.
Similar dependence of on $\nicefrac{\rho}{\gamma}$ is obtained by the Ellipsoid method~\citep{Karmarkar84},
and Vaidya Cutting Plane method~\citep{DBLP:journals/mp/Vaidya96},
that use $O(d^2\log(\nicefrac{\rho}{\gamma}))$ and $O(d\log(\nicefrac{d\rho}{\gamma}))$ iterations respectively,
when used to find a linear classifier that is consistent with a realizable sample with margin $\nicefrac{\rho}{\gamma}$ in $\R^d$.

The upper bound in the above theorem depends on the dimension,
which is often avoided in bounds that depend on the margin.
It is therefore natural to ask whether there is a bound that depends only on the margin.
We show that it is impossible:
\begin{theorem}[Theorem~\ref{thm:lbmargin}, informal version]
There is a sample of $n$ unit vectors in $\R^{n+1}$ with $\Omega(1)$ margin that require $\Omega(n)$ label and comparison queries to recover all the labels.
\end{theorem}

\paragraph{Statistical learning.}
Using standard arguments, we translate the results above to the statistical setting and get bounds
on the sample and query complexity:
the algorithmic results above directly give a learning algorithm
for realizable distributions with sample complexity
$n(\eps,\delta)=O\left(\frac{d +\log(1/\delta)}{\eps} \right)$
and query complexity approximately logarithmic in $n(\eps,\delta)$,
where $\eps$ is the error and $1-\delta$ is the confidence of the algorithm.
(See Section~\ref{sec:ub-bit} for the bit complexity and Section~\ref{sec:ub-margin} for the margin.)
Our lower bounds translate analogously to realizable distributions
that require $\Omega(1/\eps)$ queries for achieving error $\eps$ and constant confidence (say $1-\delta=5/6)$.
See Section~\ref{sec:lb} for details.

\subsubsection{Inference dimension}
\label{sec:intinf}

Our results for learning half spaces rely on a common combinatorial property that we describe next.
Let $X$ be a set and $H$ a concept class where each concept is of the form $\sign(f(x))$ for $f:X \to \R$.
For example, $X$ may be a finite set $X=\{x_1,\ldots,x_n\} \subseteq \R^d$ and $H$ the class of all half spaces with margin at least $1/100$ with respect to $X$; or
$X=\{0,1\}^d$ and $H$ is all half spaces; or $X=\R^d$ and $H$ a class of (signs of) low degree polynomials; etcetera.

Let $S\subseteq X$ be an unlabeled sample.
An {\em $S-$query} is either a label query regarding some $x\in S$,
or a comparison query regarding $x_1,x_2\in S$.
Namely,  the allowed queries are ``$f(x) \ge 0?$" (label query)
or ``$f(x_1) \ge f(x_2)$" (comparison query).
For $x\in X$ and $c\in H$, let
\[S\yields{f} x\] denote the statement that
the comparison and label queries on $S$ determine the label of $x$,
when the learned concept is $c=\sign\bigl(f(x)\bigr)$.

The \emph{inference-dimension} of $(X,H)$ is the smallest number $k$
such that for every $c = \sign(f(x)) \in H$,
and every $S \subset X$ of size at least $k$, there exists $x \in S$ such that
\[S\setminus\{x\}\yields{f} x.\]
In other words, if the inference dimension is $k$
then in every sample of size $k$ or more,
there is a point whose label can be inferred
from the label and comparison queries on the remaining points.

For example the inference dimension of $(X,H)$,
where $X=\R$, and $H$ is the class of threshold functions is 3:
indeed, to see it is at most 3,
note that if all 3 points have the same label,
then the label of the midpoint can be inferred from the other two labels,
and if not all points have the same label then the midpoint
and, say the right point have the same label, and so the label of the right
point can be inferred from the other 2 labels.
In this example comparison queries are not required.
Another example,
which requires comparison queries,
is where $X=\R^2$ and $H$ is the class of half-planes.
Here the inference dimension\footnote{One can show the inference dimension here is 5.} is  at most 7:
indeed, in any sample of at least 7 points
there are 4 points with the same label,
and the label of one of these points
can be inferred from the other 3
(see Figure~\ref{fig:plane} and Section~\ref{sec:example}).

The next Theorem shows that inference dimension captures the query-complexity
in active learning with comparison queries.
It is worth noting that in the classical setting of active learning,
when only label queries are allowed, the inference
dimension specializes to the the \emph{star dimension}~\citep{hanneke15a},
which similarly captures the (worst-case) query complexity
in this setting.

\begin{theorem}\label{thm:compinf}
Let $k$ denote the inference dimension of $(X,H)$. Then:
\begin{enumerate}
\item There is an algorithm that reveals the labels of any realizable
sample of size $n$ using at most $O(k \log k \log n)$ queries.
\item Any algorithm that reveals the labels of any realizable sample
of size $k$ must use $\Omega(k)$ queries in the worst-case.
\end{enumerate}
\end{theorem}
The upper bound (the first item) is a corollary of Theorem~\ref{thm:boosting},
and the lower bound is a corollary of Theorem~\ref{thm:infdim}.
Both Theorems are discussed in the next subsection.
While the lower bound is relatively straight forward,
our derivation of the upper bound requires several steps,
which we summarize next.
\begin{enumerate}
\item[(i)] {\bf Low inference dimension $\implies$ weak confident learner}:
we first show that if the inference dimension is at most $k$
then there is a \emph{weak confident learner}
for $(X,H)$ with query complexity $O(k\log k)$ (see Lemma~\ref{lem:weak}).
A confident learner is a learning algorithm
that may abstain from predicting on some points
$x\in X$, but must be correct on every point
where it does not abstain.
A weak confident learner is a confident learner
that with constant probability does not abstain on a constant fraction
of $X$ (see Section~\ref{sec:boosting} for a formal definition).
\item[(ii)] {\bf Boosting the weak confident learner}: once a weak
confident learner is derived, we transform it into the desired learning algorithm
using a simple boosting argument.
\end{enumerate}

While the boosting part is rather standard,
showing that low inference dimension implies a weak confident learner
relies on a symmetrization argument.
This symmetrization argument can be replaced by a more standard
sample compression argument, however this would result
in a suboptimal query complexity bound.

We get the following corollary of Theorem~\ref{thm:compinf} in the statistical setting:
\begin{corollary}\label{cor:compstat}
Let $k$ denote the inference dimension of $(X,H)$.
\begin{enumerate}
\item Let $n(\eps,\delta)$ denote the passive sample complexity of learning $(X,H)$ with error $\eps$ and confidence $1-\delta$.
There is an algorithm that learns $(X,H)$ with sample complexity $n(\eps,\delta)$
and query complexity $O\bigl(k\log k\log(n(\eps,\delta))\bigr)$.
\item Any algorithm that learns $(X,H)$ with error at most $\eps=\nicefrac{1}{k}$ and confidence at least $\nicefrac{5}{6}$
must use at least $\Omega\bigl(\nicefrac{1}{\eps}\bigr)$ queries for some realizable distribution.
\end{enumerate}
\end{corollary}
This Corollary follows from Corrolaries~\ref{cor:statupper} and~\ref{cor:statlb}.

\paragraph{A dichotomy.}
The passive sample complexity of $(X,H)$ is $n(\eps,\delta) = \Theta(\frac{d+\log(\nicefrac{1}{\delta})}{\eps})$,
where $d$ is the VC-dimension of $H$~\citep{DBLP:journals/corr/Hanneke15}.
Thus the above corollary implies a dichotomy:
if the inference dimension is finite then the query complexity
is logarithmic in $1/\eps$, and if it is infinite
then the query complexity is $\Omega(1/\eps)$.

The results presented in the previous section regarding
learning half spaces with comparison and label queries
are derived by analyzing the inference-dimension of the relevant classes,
which we sketch next.

\paragraph{Sketch of upper bounding the inference dimension of half spaces.}
Our upper bounds in terms of margin
and bit-complexity follow a similar outline.
We next give a rough sketch of the arguments
in order to convey their flavor.

Consider the case where every instance $x\in X$
has bounded bit-complexity, say $X \subset [N]^d$
for some bounded $N\in\N$.
We wish to show that for a sufficiently
large $Y\subseteq X$, and every half~space $c=\sign(f)$
 there is some $x\in Y$
such that
\[Y\setminus\{x\}\yields{f} x.\]
By removing at most half of the elements of $Y$,
we may assume that $c$ is constant on $Y$,
without loss of generality assume that
$c(x)=+1$, for every $x\in Y$.
Let $x_1,x_2,\ldots$ be an ordering
of  the elements of $Y$ such that
\[f(x_1) \leq f(x_2) \leq \ldots\]
The first observation is that it suffices
to show that there exists $i_0$
such that $x_{i_0} - x_1$ is a nonnegative
linear combination of the $x_{i} - x_j$, where $i>j$ (see Claim~\ref{c:cone}).

The existence of such an $i_0$
is achieved by a pigeon hole argument
showing that if $Y$ is sufficiently large then there are two
distinct linear combinations of the $x_{i+1}-x_i$'s with coefficients from $\{0,1\}$
that yield the same vector:
\[\sum_{i=1}^{\nicefrac{k}{2}}{\beta_i(x_{i+1}-x_i)} = \sum_{i=1}^{\nicefrac{k}{2}}{\gamma_i(x_{i+1}-x_i)},\]
Then a short calculation yields that the maximal $j$ such that $\beta_j\neq \gamma_j$
serves as the desired $i_0$.

The upper bound in terms of margin
is technically more involved.
Instead of nonnegative linear combinations,
one can consider linear combinations
such that every coefficient is at least $-\gamma$,
where $\gamma$ is the margin,
and the pigeon hole argument
is replaced by a volume argument.

\paragraph{Sketch of lower bounding the inference dimension of half spaces.}
Our lower bounds are based on embedding the class $\bigl\{\emptyset,\{i\} : 1\leq i\leq  n \bigr\}$
as half spaces in a way that the $n+1$ half spaces
induce the same ordering on the $n$ points in terms
of distance from the boundary.
Comparison queries are useless for such an embedding, which implies that the inference dimension is at least $n$,
Indeed, consider the half space $c_\emptyset$ that corresponds to the empty-set;
then for any subset $Y$ of at most $n-1$ points $c_\emptyset$ is indistinguishable
from $c_{\{i\}}$, the half space corresponding to $\{i\}$, where~$i\notin Y$.

\subsubsection{General framework}
\label{sec:intgen}

We next describe how the notion of inference-dimension,
as well as Theorem~\ref{thm:compinf} extend to settings
where the additional queries are not necessarily comparison queries.

Consider an interactive model, where the learning algorithm
is allowed to use additional queries from a prescribed set of queries $\Q$.
More formally, let $S$ be the unlabeled input sample.
In pool-based active learning,
the algorithm may query the label of any point in $S$.
Here, the algorithm
is allowed to use additional queries from a set $\Q=\Q(S)$
(we stress that the allowable queries depend on the input sample).
For example, in the setting discussed in the previous section,
$\Q(S)$ contains all comparison queries among points in $S$.
Another example, which is used by crowd-sourcing algorithms~\citep{DBLP:conf/icml/TamuzLBSK11}
involve 3-wise queries of the form
``Is $x_2$ more similar to $x_1$ than to $x_3$?''.

\paragraph{Upper bound.}
The next ``boosting result'' generalizes the upper bound
from Theorem~\ref{thm:compinf} and shows that
if there are not too many queries in $\Q(S)$,
or alternatively if there is an algorithm
that infers the answers to the queries
in $\Q(S)$ using a few queries, then
it is possible to reveal all labels using
a logarithmic number of queries:
\begin{theorem}[Restatement of Theorem~\ref{thm:boosting}]
Let $k$ denote the inference dimension of $(X,H)$.
Assume that there is an algorithm that, given a realizable sample $S$
of size $n$ as input, uses at most $q(n)$ queries and
infers the answers to all queries in $\Q(S)$
and all labels in $S$.
Then there is a randomized algorithm that infers
all the labels in $S$ using at most
\[2q(4k)\log n\]
queries in expectation.
\end{theorem}
For example, in the setting of comparison queries,
$q(n) = n + n\log n$ queries suffice to
infer all comparison queries and label queries
in $\Q(S)$, and thus, the upper bound in
Theorem~\ref{thm:compinf} follows from the above theorem.

An interactive algorithm that infers all labels using a few queries
can be combined with any passive learner
by first using the interactive learner to reveal all
labels of the input sample, and then apply the passive learner
on the labeled sample. Thus, we get:
\begin{corollary}\label{cor:statupper}
Let $(X,H)$ be as in Theorem~\ref{thm:boosting},
and let $n(\eps,\delta)$ denote the (passive)
sample complexity of learning $(X,H)$ with error $\eps$
and confidence $1-\delta$.
Then there exists an algorithm that learns $H$
with sample complexity $n(\eps,\delta)$,
and query complexity $O( q(4k)\log n(\eps,\delta) )$.
\end{corollary}

\paragraph{Lower bound.}
The next lower bound on the query complexity
generalizes the one from Theorem~\ref{thm:compinf}.
It demonstrates a property of comparison and label queries
which suffices for the lower bound in Theorem~\ref{thm:compinf} to hold.
Call an additional query \emph{$t$-local},
if its answer is determined by $f(x_1),\ldots,f(x_t)$ for some $x_1,\ldots,x_t \in X$.
For example, every label-query
is $1$-local and every comparison query is $2$-local.
The following Theorem extends
Theorem~\ref{thm:compinf} to this setting:
\begin{theorem}[Restatement of Theorem~\ref{thm:infdim}]
Assume that the inference dimension of $(X,H)$ is $>k$.
Assume that every additional query is $t$-local,
and that for every sample $S$
the set of allowable queries $\Q(S)$
is the set of all queries that are determined
by subsets (of size at most $t$) of $S$.
Then any algorithm that reveals the labels of any realizable sample
of size $k$ must use $\Omega(k/t)$ queries in the worst-case.
\end{theorem}
In the statistical setting, we get the following Corollary:
\begin{corollary}[Restatement of Corollary~\ref{cor:statlb}]
Set $\eps=\frac{1}{k}$, $\delta = \frac{1}{6}$.
Then any algorithm that learns $(X,H)$
with error $\eps$ and confidence $1-\delta$
must use $\Omega(\nicefrac{1}{t\eps})$ queries.
\end{corollary}

\subsection{Related work}
\label{sec:related}

Studying statistical learning where the learner has
access to additional queries was considered by various works.
A partial list includes:~\cite{Angluin87,Baum91,LangBaum92,Turan93,Jackson1997,KwekP98,BlumCGS98,BshoutyJT04,Feldman09,Feldman09b,Nowak11,DBLP:journals/corr/ChenHK16}.
Many of these focus on the case where the additional queries
are \emph{membership queries}. We discuss some of these results in Section~\ref{sec:additional}.

In the context of active learning, which is considered in this paper,
\cite{BalcanHanneke11} considered
additional queries of two types:  {\em Class conditional queries} and
{\em Mistake queries}, in the first type the learner provides the annotator
with a list of examples, and a label and asks her to point out an example in this list
with the given label. In the second type, the learner gives
the oracle a list of examples with proposed labels
and she replies whether it is correct or points out a mistake.
Note that these queries may have more than two answers,
which is different than binary queries, which are considered in this paper.

\cite{DBLP:conf/kdd/WauthierJJ12}
give an active learning algorithm for clustering using pairwise similarity queries,
and~\cite{DBLP:conf/nips/AshtianiKB16}
gives a clustering algorithm that uses queries that ask whether two elements belong to the same cluster.
\cite{DBLP:conf/icml/VikramD16} consider clustering
in an interactive setting where the algorithm
may present the annotator 
a clustering of a $O(1)$ size subset of the domain,
and the annotator replies whether the target-clustering
agrees with this partition, and points out
a difference in case it does not.

\cite{DBLP:conf/nips/JamiesonN11} give an active ranking algorithm
from pairwise comparisons.
They consider a setting that bears resemblance
with ours:
their goal is to find the ranking
among a sample of points in $\R^d$,
where the ranking is determined according
to the euclidean distance from a fixed, unknown reference point.
They present an algorithm that use an expected number of
$O(d\log n)$ comparisons when the ranking is chosen uniformly at random.

{Recently, \cite{xu2017noise} considered a similar setting
to ours. They also study active classification with additional comparison queries,
but they focus on minimizing the total number of label queries,
while the number of comparison queries can be large (more than $1/\eps$,
where $\eps$ is the error). In contrast, we study when it is possible to 
achieve total number of queries which is logarithmic in $1/\eps$.}

\subsection{Discussion and suggested future research}
\label{sec:future}

We next offer few potential directions for future research.

\subsubsection{Other types of additional queries}
\label{sec:additional}
\paragraph{Relative queries}
It is natural to study other types of additional queries,
and when do they enhance the learning process.
Comparison queries belong to a type of queries,
called relative queries, which is popular in many applications
involving human annotators.
A popular example of relative queries, which are used in practice,
take 3 data points $x_1,x_2,x_3$ and queries
``is $x_2$ more similar to $x_1$ than to $x_3$?''.
Such queries were studied in various works, including~\cite{DBLP:conf/nips/XingNJR02,DBLP:conf/nips/SchultzJ03,DBLP:journals/jmlr/AgarwalWCLKB07,DBLP:conf/mir/McFeeL10,DBLP:journals/kais/HuangYC11,DBLP:conf/icml/TamuzLBSK11,DBLP:conf/ijcai/QianWWLYD13}.
%

\paragraph{Membership queries from generative models}
A natural and well-studied type of queries is \emph{membership queries};
thus, the learner is allowed to query the labels of any point in the domain,
including points outside the input sample.
Using membership queries to enhance learning
was considered in the PAC model~\citep{Baum91,Turan93,Jackson1997,KwekP98,BshoutyJT04,Feldman09,Feldman09b,DBLP:journals/corr/ChenHK16}.
However, as an experimental work by \cite{LangBaum92}
suggested, querying the labels of arbitrary points
may result in a poor practical performance.
A possible explanation is
that some domain points represent noise;
for example, if the data distribution is supported on some low dimensional manifold,
then querying outside it may make no sense to the human annotator.
Consequently, various restrictions on this model were
studied~\citep{BlumCGS98,SloanT97,AwasthiFK13,Bary-WeisbergDS16}.
One option, studied by~\cite{AwasthiFK13,Bary-WeisbergDS16})
is to restrict the type of membership queries by considering only local queries,
within a small neighborhood of the input sample.
Another possible restriction is to consider only membership queries of points
that are sums and differences of a few sample points
(e.g.\ a comparison query is a membership query
on the difference of the two compared points in the case of half spaces).

An alternative direction for dealing with querying noise
is to use a generative model for constructing data outside of the input sample:
imagine we have an access to a generative model (e.g.\ Generative Adversarial Networks~\citep{DBLP:conf/nips/GoodfellowPMXWOCB14})
$g$ that gets as an input any point $s\in [0,1]^d$, and outputs
a (non-noisy) unlabeled example $g(s)\in X$.
This gives an access to a large pool membership queries
the learning algorithm can use.

From a theoretical perspective, this motivates the study
of active learning with additional membership queries,
where the learned concept is some $c:[0,1]^d\to\{\pm 1\}$ of low complexity
(say $c$ is the sign of a low degree polynomial
or a small network, etc.).
It is worth noting that the case where $c$ is a half space
was studied in the discrete geometry
community in the context of linear decision trees~\citep{DBLP:journals/jacm/Heide84,DBLP:journals/iandc/Meiser93,DBLP:journals/ipl/Liu04,DBLP:conf/esa/CardinalIO16,DBLP:journals/corr/EzraS16}.
The currently best known upper bound are due to~\cite{DBLP:journals/corr/EzraS16},
who give an algorithm that reveals all labels of a realizable
sample of size $n$ using $\tilde O(d^2\log d\log n)$ membership queries
(the best known lower bound is the information theoretical $\Omega(d\log n)$,
which applies more generally for any type of additional yes/no queries).

\subsubsection{Noisy comparisons}
It is natural to extend our algorithms by making them robust
in the presence of \emph{noisy comparisons}.
Imagine, for example that $f(x_1)=1.001$, and $f(x_2)=1.0001$.
Our algorithms are based on the assumption that the answer
to the query ``$f(x_1)\geq f(x_2)$?'' will always be ``yes''.
It is plausible to consider a noise model
that may give the wrong answer in such cases where $f(x_1) \approx f(x_2)$.
One such model is suggested by \cite{Bradley52}.
According to this model, when\footnote{if $f(x_1), f(x_2) < 0 $ then it is ``yes'' with probability $\frac{f(x_2)}{f(x_1) + f(x_2)}$.}
$f(x_1), f(x_2) >0 $ then the answer to the query ``$f(x_1)\geq f(x_2)$?'',
is ``yes'' with probability $\frac{f(x_1)}{f(x_1) + f(x_2)}$.

\subsubsection{Streaming-based interactive learning with comparison queries}
This paper focused on the pool-based model.
Another natural and well studied model
is the streaming-based model.
Here, the algorithm gets to observe the
unlabeled examples one-by-one in an online fashion,
and has to decide whether to query the current example.
Consider a setting in which the algorithm
has a bounded memory in which
it may store $M$ past examples.
Upon receiving a new example
it may query its label and/or
compare it with any of the saved examples.
In the realizable setting, a possible goal
is to minimize the number of queries and
the memory size while maintaining
a (partial) hypothesis that is correct
on all past examples.

Consider for example the class of threshold function over $\R$.
Given an infinite sequence of independent unlabeled examples $x_1,x_2,\ldots$
drawn from an unknown distribution over $\R$.
A simple online algorithm with memory $M=2$
saves the $2$ examples $x,y$ where the sign-change occurs.
Upon receiving a new example $x_n$, it checks
whether $x_n$ is between $x$ and $y$;
if it is then it queries its label and replaces
it with $x$ or $y$ (according to the label of $x_n)$,
otherwise the label of $x_n$ can be inferred and the algorithm
moves to the next example.
One can show that this algorithm makes an expected
number of $\log n$ queries after observing the first $n$ examples.
One can show that a slightly more complicated
algorithm, which uses comparison queries and memory $M\leq 10$,
has a similar behaviour for half-planes in $\R^2$.
We leave as an open question whether
similar algorithms exist for arbitrarty classes $(X,H)$
with low inference dimension.

%

\subsubsection{Exactly learning threshold functions}
\label{sec:thresholds}

Consider the class of thresholds functions over the $N$-dimensional hypercube $\{0,1\}^N$.
By Lemma~\ref{lem:bitcomp}, the inference dimension of
this class is $O( N \log N)$.
Therefore, Theorem~\ref{thm:bit} implies that
the $2^n$ labels of any threshold function
can be revealed using at most $O(N^2\log^2 N)$ queries
in expectation.
This demonstrates the additional power
of comparison queries over label queries:
indeed, standard arguments show that
if only label queries are allowed
then at least $2^{\Omega(N)}$ queries are needed.

While just $O(N^2\log^2 N)$ queries suffice to reveal
all $2^N$ labels,
it remains open whether this can be done efficiently, in a $poly(N)$ time.
A naive implementation of the algorithm
from Theorem~\ref{thm:bit} 
would check at each step for each point in $\{0,1\}^N$
whose label is still no known,
whether its be inferred by the queries performed in this step.
This takes at least $2^N$ such checks.
We leave the question, of efficiently learning all $2^N$
labels of an unknown threshold function as an open question for future work.

Another comment that seems in place is that 
the $O(N^2\log^2 N)$ upper bound is tight up to the $\log^2 N$ factor:
indeed, it is known that there are some $2^{\Theta(N^2)}$
threshold functions~\citep{DBLP:conf/ifip/GotoT62}, and therefore
an information theoretic lower bound of $\Omega(N^2)$ queries holds even if the algorithm
is allowed to use arbitrary yes/no queries.
While it remains open whether the information theoretic barrier
can be achieved using comparison queries, 
it is possible shave one log factor from the upper bound
and achieve $O(N^2\log N)$ queries in expectation,
as we sketch next. This follows since: 
(i)   the number of linear orders over a set $\{x_1,\ldots.x_m\}\subseteq\R^N$ 
that are induced by a linear function is at most some\footnote{if $f_1,f_2:\R^N\to\R$ are linear 
and induce a different linear orders on $\{x_1,\ldots,x_m\}$, then $\sign(f_1(x_i-x_j))\neq\sign(f_2(x_i-x_j))$ for some $i,j$, and so the number of distinct orders is at most the number of sign patterns $(\sign(f(x_i-x_j)))_{1\leq i < j\leq m}$, which is at most $O(m^{2N})$.} $m^{O(2N)}$, and therefore
(ii) Fredman's algorithm~\citep{DBLP:journals/tcs/Fredman76} finds such an ordering using
$q(m) = O(m + N\log m)$ comparisons (which is better than the standard $O(m\log m)$ when $m>>N$).
Thus, in Theorem~\ref{thm:boosting} we can plug $q(4k) = O(k + N\log k )$,
instead of $q(4k) = O(k\log k)$ (recall that $k= O(N\log N$),
which yields $O(N^2\log N)$ queries in expectation.

\section{Preliminaries}
\label{sec:pre}
\paragraph{Basic definitions.}
A hypothesis class is a pair $(X,H)$,
where  $X$ is a set, and $H$ is a class
of functions $h:X\to \{\pm 1\}$.
Each function $h:X\to\{\pm 1\}$ is called \emph{a hypothesis},
or \emph{a concept}.
In this paper we study classes $H=H_{\F}$ of the form
\[H = \{\sign(f) : f\in\F\},\]
where $\F = \{f:X\to\R\}$ is a class of real-valued functions,
{and $\sign(f)(x)=\sign\bigl(f(x)\bigr)\in\{\pm 1\}$ is equal $+1$ if and only if $f(x)\geq 0$.}
For example, when $\F$ is the class of $\R^d\to\R$ affine  functions
then $H_{\F}$ is the class of $d$-dimensional half spaces.
Other examples include neural-nets, low degree polynomials, and more.
The reason we ``remember'' the underlying class $\F$  is because we will use it to define comparison queries.

An \emph{example} is a pair $(x,y)\in X\times \{\pm 1\}$.
A \emph{labeled sample} $\ls$ is a finite sequence of examples.
We denote by $\us$ the \emph{unlabeled sample}
underlying $\ls$ that is obtained by removing the labels from the examples in $\ls$.
We will sometimes abuse notation and treat
$\us$ as a subset of $X$ (formally it is a sequence).
Given a distribution $D$ on $X\times \{\pm 1\}$,
the (expected) \emph{loss} of a hypothesis $h$ is defined by:
\[L_D(h) \eqdef \E_{(x,y)\sim D}\bigl[1_{h(x)\neq y}\bigr].\]
Given a labeled sample $\ls$, the \emph{empirical loss} of $h$
is defined by:
\[L_{\ls}(h) \eqdef \frac{1}{\bigl\lvert\ls\bigr\rvert}\sum_{(x,y)\in \ls}1_{h(x)\neq y}.\]

A distribution $D$ on $X\times \{\pm 1\}$ is \emph{realizable} by $H$
if there exists $c\in H$ with $L_D(c)=0$.
We will sometimes refer to such a $c$ as the ``learned concept''.
A labeled sample $\ls$ is \emph{realizable} by $H$ if there exists $c\in H$
with $L_{\ls}(c)=0$.

\paragraph{Passive learning (PAC learning).}
A learning algorithm is an (efficiently computable)
mapping that gets a sample as an input
and outputs a hypothesis.
An algorithm $A$ learns $H$ if there exists a
sample complexity bound $n(\eps,\delta)$,
such that for every realizable distribution $D$,
given a labeled sample $\ls$ of size $n\geq n(\eps,\delta)$, $A$ outputs $h=A\bigl(\ls\bigr)$ such that:
\[\Pr_{\ls\sim D^n}\bigl[L_D(h) > \eps\bigr]\leq \delta.\]
The parameter $\eps$ is called the error of the algorithm, and $1-\delta$ is called the confidence of the algorithm
We will assume throughout that a learning algorithm for $H$
also receives $\eps,\delta$ as part of the input and can compute $n(\eps,\delta)$.

\subsection{Active learning}
It is helpful to recall the framework of active learning before
extending it by allowing additional queries.
A (pool-based) active learning algorithm has an access to the unlabeled
sample $\us$ underlying the input labeled sample $\ls$.
It queries the labels of a subsample of it, and outputs a hypothesis.
The choice of which subsample $A$ queries may be adaptive.
Each active learning algorithm is associated with
two complexity measures:
(i) the sample-complexity $n(\eps,\delta)$,
is the number of examples required to achieve error
at most $\eps$ with confidence at least $1-\delta$
(like in the passive setting),
and (ii) the query-complexity (also called label-complexity) $q(\eps,\delta)$,
is the number of queries it makes.

In the process of active learning, it is natural
to distinguish between points whose
label can be inferred and points for which
there is uncertainty concerning their label.
It is therefore convenient to consider
partial hypotheses.
A \emph{partial hypothesis}
is a partially labeled hypothesis $h:X \to (Y \cup \{?\})$,
where if $h(x)="?"$ it means that $h$
abstains from labeling $x$.
We extend the $0/1$ loss function to "?",
such that abstaining is always treated as a mistake.
The \emph{coverage} of $h$ with respect to a distribution $D$ is defined as
\[C_D(h) \eqdef \Pr_{x\sim D}\bigl[h(x) \neq ?\bigr],\]
and the \emph{empirical-coverage} of $h$ with respect to a sample $S$
is defined as
\[C_S(h) \eqdef \frac{\bigl\lvert\{x\in S_X : h(x)\neq ? \}\bigr\rvert}{\lvert S\rvert}.\]
Since abstaining counts as error it follows that $C_D(h)\leq 1-L_D(h)$ for every partial hypothesis $h$,
and every distribution $D$.

\paragraph{Confident algorithms.}
A learning algorithm is \emph{confident}
if it outputs a partial hypothesis that is correct on all points where it does not abstain:
\begin{definition}[Confident learning algorithm]
A learning algorithm $A$ is \emph{confident}  with respect to a class $(X,H)$
if it satisfies the following additional requirement.
For every realizable labeled sample $\ls$ that is consistent with a learned concept $c\in H$,
the output hypothesis $h\eqdef A\bigl(\ls\bigr)$ satisfies that whenever $h(x) \ne ?$ then $h(x)=c(x)$.
\end{definition}
Thus, if $A$ is confident then
$C_D(h)= 1-L_D(h)$ for every distribution $D$,
where $h=A\bigl(\ls\bigr)$.
Therefore, in the context of confident learners
we will only discuss their coverage, and omit
explicit reference to their error.
For example a class $(X,H)$
is learned by a confident learner $A$
if there exists a sample complexity $n(\eps,\delta)$
such that for every realizable distribution $D$,
given a labeled sample $\ls$ of size $n\geq n(\eps,\delta)$,
$A$ outputs $h=A\bigl(\ls\bigr)$
such that
\[\Pr_{\ls\sim D^n}\bigl[C_D(h) < 1- \eps\bigr]\leq \delta.\]

%

\subsection{Interactive learning with additional queries}
Consider an extension of the active learning setting,
by allowing the learning algorithm
to use additional queries from a prescribed set of queries $\Q$.
An additional query is modeled as a boolean function $q:\F \to\{True,False\}$.
We stress that the query may depend on the function $f$
that underlies the learned concept $c=\sign(f)$ (e.g.\
comparison queries, see below).
In this setting, given an input sample $\ls$,
the algorithm is given access to $\us$,
the unlabeled sample underlying $\ls$,
and is allowed to:
\begin{itemize}
\item query the label of any of point in $\us$,
\item query an additional query $q$
from a prescribed set of queries $\Q(\us)$.
\end{itemize}
We stress that the set of allowable queries $\Q(\us)$ depend on the input sample $\ls$.
For example, a comparison query on $x_1,x_2$
is the query ``$f(x_1) \leq f(x_2)$?'', where $c=\sign(f)$
is the learned concept.
An {\em answered query} in a pair $(q,b)$, where $q$ is a comparison/label-query,
and~$b$ is a possible answer to $q$.
For example, $\bigl(\text{``}f(x_1)\leq f(x_2)?\text{''},True\bigr)$ is an answered
comparison query, and $\bigl(\text{``}c(x_1)=\ ?\text{''},-1\bigr)$ is
an answered label query.
The standard notions of \emph{version space} and \emph{agreement set} are naturally extended to this context:
let
\[\bar Q = \bigl((q_1,b_1),(q_2,b_2),\ldots (q_m,b_m)\bigr)\]
be a sequence of answered queries.
Define the {\em version space}, denoted by $V(\bar Q)$, as the set of hypotheses in $H$
that are consistent with $\bar Q$:
\[V(\bar Q) \eqdef \bigl\{ h\in H : q_i(h) = b_i,~ i=1,\ldots,m\bigr\},\]
and the confidence region, $\conf(\bar Q)$, as the agreement set of $V(\bar Q)$:
\[\conf(\bar Q) \eqdef \bigl\{x\in X: \text{all hypotheses in }V(\bar Q)\text{ agree on }x\bigr\}.\]

\section{Inference Dimension}\label{sec:infdim}

Let $(X,H)$ be a hypothesis class, where $H=H_{\F}$ for some class
of real functions $\F$, and let $\Q$ be a set of additional queries.
For $\us\subseteq X$, $x\in X$, $f \in \F$ and $c=\sign(f)\in H$, let
\[\us\yields{f} x\]
denote the statement that there exists a sequence $\bar Q$
of answered label queries of $S$ and/or additional queries from $\Q(\us)$
that determine the label of $x$, when the learned concept is $c$.
Namely, that $x\in\conf(\bar Q)$.

\begin{definition}[Inference dimension]
The inference dimension of $(X,H)$
is the minimal number $k$ such that
for every $\us \subseteq X$ of size $k$,
and every $c\in H$ there exists $x\in \us$ such that
\[\us\setminus\{x\}\yields{f} x.\]
If no such $k$ exists then the inference dimension of $(X,H)$ is defined as $\infty$.
\end{definition}

\subsection{Upper bound}
\label{sec:boosting}
\begin{theorem}[Boosting]\label{thm:boosting}
Let $k$ denote the inference dimension of $(X,H)$.
Assume that there is an algorithm that, given a realizable sample $\ls$
of size $n$ as input, uses at most $q(n)$ queries and
infers the answers to all queries in $\Q(\us)$
and all labels in $\ls$.
Then there is a randomized algorithm that infers
all the labels in $\ls$ using at most
\[2q(4k)\log n\]
queries in expectation.
\end{theorem}

We prove Theorem~\ref{thm:boosting} in two steps:
(i) First, Lemma~\ref{lem:weak}  shows that $(X,H)$ has a weak confident
learner $A$ that, given a realizable input sample of size $4k$,
uses at most $q(4k)$ queries, and outputs a partial hypothesis
$h$ with coverage $1/2$.
(ii) Then, we show that the labels of a given sample $\ls$ of size $n$
are revealed after applying $A$ on roughly $\log n$ random subsamples of~$\ls$.

\begin{lemma}[Weak confident-learning]\label{lem:weak}
Let $k$ denote the inference dimension of $(X,H)$.
Then there exists a confident learner for $(X,H)$
that is defined on input samples of length $4k$,
makes at most $q(4k)$ queries, and
has coverage $\ge 1/2$ with probability $\ge 1/2$. That is,
for any distribution $D$ over $X$,
$$
\Pr_{S \sim D^{4k}}\bigl[C_D(h) \ge 1/2\bigr] \ge 1/2,
$$
where $h$ is the output hypothesis of the algorithm.
\end{lemma}
\begin{proof}
The learner is defined as follows.
Given a realizable input sample $S=\bigl(\bigl(x_i,c(x_i)\bigr)\bigr)_{i=1}^{4k}$,
from $D^{4k}$, the algorithm infers the answers to all queries
in $\Q(\us)$ and all labels in $\ls$
(by assumption, this can be done with $q(4k)$ queries).
It outputs the partial hypothesis $h$, which labels any $x$ whose label can be inferred from the queries. Namely:
\[h(x) \eqdef
\begin{cases}
c(x) &\us \yields{f} x\\
?    &\text{otherwise}
\end{cases}
\]
We next claim that the expected coverage of $h$, $C_D(h)$, is at least~$3/4$. This implies
that $\Pr_S[C_D(h) \ge 1/2] \ge 1/2$, which shows that the learning algorithm is a weak confident learner.

To this end we use the following observation.
\begin{obs}
For any set $Y$ of size $4k+1$, there are $x_{i_1},\ldots, x_{i_{3k+1}}\in Y$
such that for all $1 \le j \le 3k+1$ it holds that
$$
Y\setminus\{x_{i_j}\}\yields{f} x_{i_j}.
$$
\end{obs}
\begin{proof}
This follows since the inference dimension of $(X,H)$ is $k$.
Assume we already constructed $x_{i_1},\ldots,x_{i_{j-1}}$ for $j\le 3k+1$.
Let $Y'=Y \setminus \{x_{i_1},\ldots,x_{i_{j-1}}\}$. As $|Y'| \ge k$, there exists $x_{i_{j}} \in Y'$ such
that $Y'\setminus\{x_{i_j}\} \yields{f} x_{i_{j}}$. But then also $Y\setminus\{x_{i_j}\}\yields{f} x_{i_j}$.
\end{proof}

Next, we show that this observation implies that $\Ex_{S}[C_D(h)] \geq 3/4$. Clearly, we have that
\[
\Ex_{S}\bigl[C_D(h)\bigr] = \mathrm{Pr}_{(x_1,x_2,\ldots,x_{4k+1})\sim D^{4k+1}}\bigl[\{x_1,\ldots, x_{4k}\} \yields{f} x_{4k+1}\bigr].
\]
Letting $T=\{x_1,x_2,\ldots,x_{4k+1}\}$, this is the probability that $T\backslash \{x_{4k+1}\} \yields{f} x_{4k+1}$. However, by symmetry, this is the same as the probability that $T\backslash \{x_{i}\} \yields{f} x_{i}$ for any $1\leq i \leq 4k+1$. Taking the average over $i$, we have
\[
\Ex_{S}\bigl[C_D(h)\bigr] = \Ex_T\left[\frac{1}{4k+1} \big|\{i: T \backslash \{x_i\}\yields{f} x_i\}\big|\right] \ge \frac{3k+1}{4k+1} \ge \frac{3}{4}.
\]
\end{proof}

\begin{proof}[Proof of Theorem~\ref{thm:boosting}]

Let ${A}$ denote the weak confident learner from Lemma~\ref{lem:weak}.
Let $\ls$ be a realizable sample, corresponding to an unknown concept $c \in H$.
Our goal is to fully recover the labels of $\ls$.

The algorithm proceed in iterations $t=1,2,\ldots$,
At each iteration it applies $A$
on a subsample of size $4k$ of $\ls$.
Let $h_t$ denote the output hypothesis
of $A$ on iteration $t$,
and let
\[DIS_t = \{x : h_s(x) = ? \text{ for all }s\leq t\}.\]
Since $A$ is confident it follows
for any point $x\notin DIS_t$, the label $c(x)$
is equal to $h_s(x)$ for any $h_s(x)$
such that $h_s(x)\neq ?$.
As long as $DIS_t  \cap \us\neq\emptyset $,
perform the following update step.
\begin{framed}
\noindent
\underline{Update step at time $t$:}
\begin{enumerate}
\item[(1)] Let $D_t$ be the uniform distribution over $DIS_t\cap \us$. Sample $\bar R_t \sim (D_t)^{4k}$.
\item[(2)] Apply ${A}$ to $R_t$, the unlabeled sample underlying $\bar R_t$.
\item[(3)] Let ${h}_t={A}\bigl(\bar R_t\bigr)$ be the confident partial hypothesis that ${A}$ outputs on $\bar R_t$.
\item[(4)] Compute $e_t=C_{D_t}[{h}_t] = \Pr_{x \sim D_t}[{h}_t(x)\neq ?]$.
\item[(5)] If $e_t < 1/2$ then go back to step (1). Otherwise set $t\leftarrow t+1$ and continue.
\end{enumerate}
\end{framed}

Since $A$ is confident, it follows that once $DIS_t\cap \us = \emptyset$ then all the labels
of $\ls$ are revealed.

\paragraph{Query-complexity.}
In order to analyze the query-complexity of the algorithm,
first observe that since $\Pr[e_t \geq 1/2] \ge 1/2$ then in expectation,
we proceed to the next iteration after at most two samples
of $\bar R_t$. Next, if $e_t \geq 1/2$ then by definition
$|DIS_{t+1}\cap S_X| \le  \lvert DIS_{t}\cap S_X\rvert/2$.
Thus, we only apply the update step at most
$t_{max} \le 2\log n$ many times.
It follows that the expected query-complexity is at most $2q(4k) \log n$.
\end{proof}

\paragraph{Computational complexity.}
The algorithm derived in Theorem~\ref{thm:boosting}
has expected running time of $O(T_{\text{update}} \log n)$,
where $T_{\text{update}}$ is the running time of the update step.
In every update step the algorithm makes $q(4k)$
queries and determines $e_t$, by checking
for each unlabeled point, whether its label
can be inferred by the queries performed in this step.
Assume that testing this for each point takes take $T_{\text{infer}}$.
So, 
$$
T_{\text{update}} = O \left(q(4k) + n \cdot T_{\text{infer}} \right)
$$
and the total running time is
$$
T_{\text{total}} = O \left(\left( q(4k) + n \cdot T_{\text{infer}} \right) \log n \right).
$$
For example, when the hypothesis class is half spaces in $\R^d$, and the 
the set of additional queries is comparisons, the total running time is polynomial in $n$.
This is since $q(4k) = O(k\log k)$ (by sorting), and since checking if the label of a point
is inferred by a set of label and comparison queries can be phrased as a linear program
and solved in polynomial time (see Claim~\ref{c:cone}).

\subsection{Lower bound}

Next, we show that if the inference dimension is large then many queries are needed to infer all the labels.
We further assume that every query is $t$-local, in the sense that it depends on $f(x_1),\ldots,f(x_t)$ for some $x_1,\ldots,x_t \in X$.
We set of allowable queries $\Q(\us)$ to be all queries that are associated with subsets of $\us$ of size $t$.

Let $(X,H)$ be a hypothesis class with inference dimension $>k$, for some $k \ge 3$.
This means that there exists $Z\subseteq X$ of size $k$ and a concept $c\in H$
such that for every $z\in Z$ there is $c_z\in C$ with $c_z(z)\neq c(z)$,
but $c_z(x)=c(x)$ for all $x \in Z \setminus \{z\}$, and moreover, 
$q(c)=q(c_z)$ for every query $q\in \Q\bigl(Z\setminus\{z\}\bigr)$.

\begin{theorem}\label{thm:infdim}
Any algorithm that reveals the labels of any realizable sample
of size $k$ must use at least $k/t$ queries in the worst-case.
\end{theorem}
\begin{proof}
Consider the realizable sample $\ls=\bigl(x,c(x)\bigr)_{x\in Z}$,
and assume that less than $k/t$ queries were used on it.
This means that there exist $z\in \us$ that is not associated
with any of the queries that were used, and hence
only queries from $\Q\bigl(Z\setminus\{z\}\bigr)$ were used.
However, this implies that both $c_z$ and $c$ are consistent
with the queries that were used, and $c_z(z)\neq c(z)$.
Thus, the label of $z$ can not be inferred by the queries that were used.
\end{proof}

\begin{corollary}\label{cor:statlb}
Let $\eps=\frac{1}{k}, \delta=\frac{1}{6}$, and let $D$ be the uniform
distribution over $Z$.
Then any learning algorithm that
makes less than $\frac{1}{t\eps}$
queries suffers a loss of $\eps$, with probability
at least~$\delta$.
\end{corollary}
\begin{proof}
Let $A$ be a learning algorithm.
Consider an adversary that picks
the secret-concept $c=c_z$ where $z \in Z$ is chosen uniformly.
Define
\[E\eqdef \bigl\{ c_x : \text{a query associated with }x \text{ was queried by } A \bigr\}.\]
If $A$ makes less than $\frac{k}{2t}$ queries
(i.e.\ $\lvert E\rvert <\frac{k}{2}$) then:
\begin{enumerate}
\item $\Pr\bigl[c_z\notin E\bigr] \geq \frac{1}{2}$, and
\item $\Pr\Bigl[L(A) \geq \frac{1}{k} ~\big\vert~ c_z\notin E\Bigr]\geq 1-\frac{2}{k}\geq \frac{1}{3}$,
\end{enumerate}
where
\[L(A)=\frac{\bigl\lvert\{x\in Z : h(x)\neq c_z(x)\}\bigl\lvert}{k}\]
is the loss of the hypothesis $h$ that $A$ outputs.
To see why the second item holds, note that since the answers to the queries
are the same for every $c_x\notin E$, it follows that the distribution of $c_z$
conditioned on $c_z \notin E$ is uniform over $E^c$ whose size is at least $\frac{k}{2}$,
and so the probability that $A$ outputs $c_i$ is at most $\frac{2}{k}$. In any other case,
the loss of $h$ is at least $\frac{1}{k}$.

Combining the above two items together yields that
\[ \Pr\Bigl[L(A) \geq \frac{1}{k}\Bigr] = \Pr\bigl[c_z\notin E\bigr]\cdot\Pr\Bigl[L(A) \geq \frac{1}{k} ~\big\vert~ c_z\notin E\Bigr]\geq \frac{1}{6}.\]
\end{proof}

\section{Interactive learning of half spaces with comparison-queries}
\label{sec:comparison}

In this section we restrict our attention to the class $H_d = \{\sign(f) : f:\R^d\to \R\}$ of half spaces in $\R^d$, where for simplicity
of exposition we consider linear functions $f$ (these correspond to homogeneous half spaces). Our results extend to the non-homogeneous
case, as non-homogeneous half spaces in dimension $d$ can be embedded as homogeneous half spaces in dimension $d+1$.
The additional queries allowed are comparison queries. That is, a label query returns the answer to $\sign(f(x))$ and a comparison query returns the answer to $f(x_1) \ge f(x_2)$.

In Subsection~\ref{sec:ub} we present our upper bounds
on the query complexity, under two natural conditions: small bit complexity, or large margin.
In  Subsection~\ref{sec:lb} we present lower bounds showing that these conditions are indeed necessary
for obtaining query complexity sub-linear in the sample complexity.

\subsection{Upper bounds}\label{sec:ub}

\subsubsection{Bit-complexity}
\label{sec:ub-bit}

Here we show that if the examples can be represented
using a bounded number of bits then comparison-queries can reduce the query-complexity.
We formalize bounded bit-complexity by assuming that $X = [N]^d$,
where $[N] = \{0,\ldots, N\}$. Note that each example can be represented by $B=d\log N$ bits.
We provide a bound on the query-complexity
that depends efficiently on$d$ and $\log N$. Variants of the arguments we use apply to
other standard ways of quantifying bounded bit-complexity.

\begin{theorem}\label{thm:bit}
Consider the class $\bigl([N]^d,H_d\bigr)$.
There exists an algorithm that reveals the labels
of any realizable input sample of size $n$ using at most
$O(k\log k \log n)$ label/comparison-queries in expectation,
where $k=O\bigl(d\log(Nd)\bigr)$.
\end{theorem}

As a consequence it follows that the hypothesis class $\bigl([N]^d,H_d\bigr)$ is learnable with
\begin{center}
sample complexity $\tilde O\bigl(\nicefrac{d}{\eps}\bigr)$
and query-complexity  $\tilde O\bigl(d\log (N) \log(\nicefrac{1}{\eps})\bigr)$,
\end{center}
where the $\tilde O$ notation
suppresses lower order terms and the usual $\log(1/\delta)$ dependence.

In order to prove the above theorem, we use Theorem~\ref{thm:compinf}
that reduces it to the following lemma.
\begin{lemma}\label{lem:bitcomp}
Let $k$ such that $2^{k/2} > 2 (k N + 1)^d$.
Then the inference dimension of the class $\bigl([N]^d,H_d\bigr)$
is at most $k$. In particular, it is at most $16d\log(4Nd)$.
\end{lemma}

\begin{proof}

Let $c = \sign(f)\in H_d$.
We will use the following claims.
\begin{obs}\label{obs:sign}
Let $x\in\mathbb{R}^d$, and $y\in\{\pm 1\}$. Then
{
\begin{align*}
yf(x) > 0 &\implies c(x) = y,\\
f(x) = 0 &\implies c(x)=+1.
\end{align*}}
\end{obs}
\begin{obs}\label{obs:ip}
Let  $x_1,x_2\in\R^d$ such that $c(x_1)=c(x_2)=y$.
Then
{
\[ \lvert f(x_2) \rvert \geq \lvert f(x_1)\rvert \iff yf(x_2-x_1) \geq 0 .\]
}
\end{obs}
\begin{claim}\label{c:cone}
Let $x_1,\ldots, x_m\in\mathbb{R}^d$, $y\in\{\pm 1\}$
such that:
\begin{itemize}
\item(i) $c(x_i) = y$ for all~$i$, and
\item(ii) $\lvert f(x_1)\rvert {\leq} \ldots {\leq} \lvert f(x_m)\rvert$.
\end{itemize}
Then, if $x\in\mathbb{R}^d$ satisfies
$x-x_{1} = \sum_{i=1}^{m-1}{\alpha_{i}(x_{i+1}-x_{i})} \text{ where } \alpha_{i}\geq 0$,
then $c(x) =y$.
In particular $\{x_1,\ldots, x_m\}\yields{f} x$.
\end{claim}
\begin{proof}
By linearity of $f$, and Observation~\ref{obs:ip}:
\[yf(x-x_1) = \sum_{i}{\alpha_{i}yf(x_{i+1}-x_i)} \geq 0.\]
Therefore $yf(x) {\geq} yf(x_1) \geq 0$,
which,  by Observation~\ref{obs:sign}, implies that $c(x) = y$ .
\end{proof}
Let $k$ be as in the formulation of the lemma,
and let $Y\subseteq [N]^d$ of size $k$.
For simplicity of exposition assume $k$ is even, and let $m=k/2$.
We need to show that there exists $x\in Y$
such that $Y\setminus\{x\}\yields{f} x$.
Without loss of generality, assume that at least $m$
of the points in $Y$ have a $1$-label.
Let $x_1,x_2,\ldots,x_{m}\in Y$ be an ordering of the
$1$-labeled points such that $\lvert f(x_1)\rvert\leq \ldots\leq \lvert f(x_{m})\rvert$.
By Claim~\ref{c:cone} it suffices to show that there exists $2 \le i^* \leq m$ such that
\begin{equation}\label{eq:1}
x_{i^*}-x_{1} = \sum_{i=1}^{i^*-2} {\alpha_i(x_{i+1}-x_i)} \text{ where } \alpha_i\geq 0.
\end{equation}

This follows by the following pigeon hole argument:
consider all possible boolean combinations of the form $\sum_{i=1}^{m-1}{\beta_i(x_{i+1}-x_i)}$
where $\beta_i\in\{0,1\}$. There are $2^{m-1}$ boolean combinations,
each of which yields a vector in $\{-mN,\ldots,mN\}^d$.
Therefore, by our assumptions on $k$ we have that
\[2^{m-1} > (2mN + 1)^d\]
there exists two boolean combinations:
\[\sum_{i=1}^{m-1}{\beta_i(x_{i+1}-x_i)} = \sum_{i=1}^{m-1}{\gamma_i(x_{i+1}-x_i)},\]
such that $\beta_i\neq \gamma_i$ for at least one $i$. Let $i_0$ be the maximal such $i$. We verify that $i^*=i_0+1$ satisfies Equation~\eqref{eq:1} above.

Assume without loss of generality that $\beta_{i_0}=0, \gamma_{i_0}=1$. Let $\alpha_i = \beta_i-\gamma_i$. Thus $\alpha_{i_0} =-1$,
$\alpha_i=0$ for all $i> i_0$ and $\alpha_i \in \{0,\pm1\}$ for $i<i_0$. So
\[\sum_{i=1}^{i_0} \alpha_i(x_{i+1}-x_i)=0.\]
Next, add $x_{i_0+1}-x_1=\sum_{i=1}^{i_0}(x_{i+1}-x_i)$ to both sides. This yields
\[x_{i_0+1}-x_{1} = \sum_{i=1}^{i_0} (\alpha_i+1)(x_{i+1}-x_i).\]
This concludes the proof as $\alpha_{i_0}+1=0$ and $\alpha_i+1 \ge 0$ for all $i<i_0$.
\end{proof}

\subsubsection{Minimal-ratio and margin}
\label{sec:ub-margin}

Let $X\subseteq \R^d$ and let $c=\sign(f)\in H_d$.
The {\em minimal-ratio} of $X$ with respect to $c$ is defined by
\[\eta=\eta(c,X)\eqdef\frac{\min_{x\in X}{\lvert f(x)\rvert}}{\max_{x\in X}{\lvert f(x)\rvert}}.\]
Here we show that it is possible
to reveal all labels using at most $\tilde O\bigl(d\log(\nicefrac{1}{\eta})\log n\Bigr)$,
where the minimal-ratio is $\eta$.
Note that the minimal-ratio is invariant under scaling and that it is upper bounded by the margin:
\begin{claim}
Let $\eta$ be the minimal-ratio of $X$ with respect to $c$,
let $\rho = \max_{x\in X}{\lvert\lvert x\rvert\rvert_2}$,
and let $\gamma$ be the margin of $X$ with respect to $c$.
Then
\[\frac{\gamma}{\rho}\leq \eta.\]
\end{claim}
\begin{proof}
Let $w$ such that  $f(x) = \ip{w,x}$, and set $w' = \frac{w}{\lvert\lvert w \rvert\rvert}$. Thus,
\[\frac{\gamma}{\rho} = \frac{\min_x\bigl\lvert \ip{w',x}\bigr\rvert}{\max_x\lvert\rvert  x \rvert\rvert_2}\leq
\frac{\min_x\bigl\lvert\ip{w',x}\bigr\rvert}{\max_x\bigl\lvert\ip{w',x}\bigr\rvert}=
\frac{\min_{x\in X}{\lvert f(x)\rvert}}{\max_{x\in X}{\lvert f(x)\rvert}} = \eta.\]
\end{proof}
Thus, the upper bound in Theorem~\ref{thm:marg+dim} below
applies when the minimal-ratio is replaced by the standard margin parameter
$\nicefrac{\gamma}{\rho}$.


Note
that there are cases where $\eta >> \nicefrac{\gamma}{\rho}$.
For example, assume $X = \{e_1,\ldots, e_d\}$
is the standard basis and $c_w\in H_d$ is determined by the normal $w=\frac{1}{\sqrt{d}}(+1,-1,+1,-1,\ldots)$. In this case, $\nicefrac{\gamma}{\rho}=\nicefrac{1}{\sqrt{d}} << 1 =\eta$.

%

We next state and prove the upper bound.
Let $X\subseteq\R^d$, and let $H_{d,\eta}\subseteq H_d$
be the set of all half spaces with minimal-ratio
at least $\eta$ with respect to $X$.

\begin{theorem}\label{thm:marg+dim}
Consider the class $(X,H_{d,\eta})$.
There exists an algorithm that reveals the labels
of any realizable input sample of size $n$ using at most
$O\bigl(k\log k \log n\bigr)$ label/comparison-queries in expectation,
where $k=O\bigl(d\log(d)\log(\nicefrac{1}{\eta})\bigr)$.
\end{theorem}

As a consequence it follows that the hypothesis class $(X,H_{d,\eta})$ is learnable with
\begin{center}
sample complexity $\tilde O\bigl(\nicefrac{d}{\eps}\bigr)$
and query-complexity  $\tilde O\bigl(d\log(\nicefrac{1}{\eta})\log(\nicefrac{1}{\eps})\bigr)$,
\end{center}
As before, the $\tilde O$ notation
suppresses lower order terms and the usual $\log(1/\delta)$ dependence.

The above theorem is a corollary of Theorem~\ref{thm:compinf}
via the following lemma, which upper bounds the inference dimension
of the class $(X,H_{d,\eta})$.

\begin{lemma}\label{lem:marg+dim}
Let $k$ such that $(k/2 +1)^d < 2^{k/2} (\eta/6)^d$.
Then the inference dimension of the class $\bigl(X,H_{d,\eta}\bigr)$
is at most $k$. In particular, it is at most $10d\log(d+1)\log(\nicefrac{2}{\eta})$.
\end{lemma}
\begin{proof}

Let $c = \sign(f)\in H_{d,\eta}$. 
We will use the following claims
(note the analogy with the claims
in the proof of Lemma~\ref{lem:bitcomp}).

\begin{obs}\label{obs:margin}
Let $x\in X$, and $y\in\{\pm 1\}$. Then
\[ \frac{yf(x)}{\max_{x'\in X}{\lvert f(x')\rvert}}\geq -\eta \implies c(x) =y.\]
\end{obs}

\begin{claim}\label{c:conemarg}
Let $x_1,\ldots, x_m\in X$, $y\in\{\pm 1\}$
such that:
\begin{itemize}
\item[(i)] $c(x_i) = y$ for all~$i$, and
\item[(ii)] $\lvert f(x_1)\rvert {\leq} \ldots {\leq} \lvert f(x_m)\rvert$.
\end{itemize}
Then, if $x\in X$ satisfies
$x-x_{1} = \sum_{i=1}^{m-1}{\alpha_{i}(x_{i+1}-x_{i})} \text{ where } \alpha_{i}\geq -\eta$,
then $c(x) =y$.
In particular $\{x_1,\ldots, x_m\}\yields{f} x$.
\end{claim}
\begin{proof}
Assume towards contradiction that $c(x) \neq y$. {Therefore, $yf(x) \leq 0$}.
By assumption, $x$ satisfies: 
\[yf(x-x_1) = \sum_{i=1}^{m-1}{\alpha_i yf(x_{i+1}-x_i)} \geq  -\eta y \sum_{i=1}^{m-1}{f(x_{i+1}-x_i)} = -\eta y f(x_m- x_1),\]
which follows by linearity of $f$ and our assumptions.
Therefore, 
\[ yf(x) \geq  yf(x_1) - \eta yf(x_m)  + \eta yf(x_1) \geq - \eta yf(x_m) , \]
which implies that 
\[0\geq \frac{yf(x)}{\max_{x'\in X}{\lvert f(x')\rvert}}\geq \frac{yf(x)}{\lvert f(x_m) \rvert} =
\frac{yf(x)}{yf(x_m) }\geq -\eta.\]
By Observation~\ref{obs:margin} this implies that $c(x)=y$,
which contradicts the assumption that $c(x)\neq y$.
\end{proof}



Let $k$ be as in the formulation of the lemma,
and let $Y\subseteq X$ of size $k$.
For simplicity of exposition assume $k$ is even,
and let $m=k/2$.
We need to show that there exists $x\in Y$
such that $Y\setminus\{x\}\yields{f} x$.
Without loss of generality, assume that at least half
of the points in $Y$ have a $1$-label.
Let $x_1,x_2,\ldots,x_{m}\in Y$ be an ordering of the
$1$-labeled points such that $\lvert f(x_1)\rvert\leq \ldots\leq \lvert f(x_m)\rvert$.
By Claim~\ref{c:conemarg} it suffices to show that there exists $2 \le i^* \leq m$
such that
\begin{equation}\label{eq:3}
x_{i^*}-x_{1} = \sum_{i=1}^{m-1} \alpha_i(x_{i+1}-x_i) \text{ where } \alpha_i \geq -\eta \text{ and } \alpha_{i^*-1}=\alpha_{i^*}.
\end{equation}
This follows by a volume argument:
let $\C$ be the convex hull of $\{x_2-x_1,x_3-x_1,\ldots,x_{m}-x_1\}$. 
For a set $A\subseteq \bigl[m\bigr]$, let\footnote{We use here the standard notation
of $u+\alpha A = \{u+\alpha a : a\in A\}$, where $u\in\R^d, \alpha\in\R,$ and $A\subseteq\R^d$.}
\[\C_A \eqdef \Bigl(\sum_{j\in A}x_{j+1} - x_j\Big) + \frac{\eta}{6}\C.\]
We claim that there are $A\neq B$ such that $\C_A\cap \C_B\neq\emptyset$.
Indeed, assume towards contradiction that the $\C_A$ are all mutually disjoint.
Note that for all $A$: (i) $\C_A\subseteq \bigl(m+\eta/6\bigr) \C$, and
(ii) $\vol(\C_A) = (\eta/6)^d\vol(\C)$.
So, if all the $\C_A$'s are mutually disjoint then
\[
(m +1)^d \vol(\C)\geq 2^{m} (\eta/6)^d\vol(\C),\]
which means that $(m +1)^d > 2^{m} (\eta/6)^d$,
which contradicts the property $k$ is assumed to satisfy.

Thus, there exist two combinations:
\[\sum_{i=1}^{m}{\beta_i(x_{i+1}-x_i)} = \sum_{i=1}^{m}{\gamma_i(x_{i+1}-x_i)},\]
such that $\beta_i \in [b_i, b_i+\gamma/6]$ and 
$\gamma_i \in [c_i, c_i+\gamma/6]$ for some $b_i, c_i \in \{0,1\}$, not all the same.
Let $i_0$ be maximal such that $b_{i_0} \ne c_{i_0}$. We will prove that Equation~\eqref{eq:3} holds for $i^*=i_0+1$.

Assume without loss of generality that
$b_{i_0}=0, c_{i_0}=1$. Define $\alpha_i = \beta_i - \gamma_i$. Thus  
$\alpha_{i_0}\in[-1-\eta/6,-1+\eta/6]$, and $\alpha_i \geq -\eta/6$ for all $i >i_0$.
Now, adding
\[ (\alpha_{i_0+1}-\alpha_{i_0})(x_{i_0+1}-x_{1})=\sum_{j=1}^{i_0}(\alpha_{i_0+1}-\alpha_{i_0})(x_{i+1}-x_{i})\]
to both sides of $\sum_{j=1}^{m}\alpha_i(x_{i+1}-x_i)=0$ gives:
\begin{align*}
&(\alpha_{i_0+1}-\alpha_{i_0})(x_{i_0+1}-x_1) =\\
&\sum_{j=1}^{i_0-1}(\alpha_i + \alpha_{i_0+1}-\alpha_{i_0})(x_{i+1} - x_i) 
+\alpha_{i_0+1}(x_{i_0+2} - x_{i_0}) 
+\sum_{j=i_0+2}^{m}\alpha_i(x_{i+1} - x_i).
\end{align*}
We conclude the proof by verifying that Equation~\eqref{eq:3} is satisfied after after dividing both sides
of the above formula by
\[\alpha_{i_0+1}-\alpha_{i_0}\geq  1 - \eta/3 \ge 1/2.\]
Indeed, the coefficients in front of $(x_{i_0}-x_{i_0-1})$ and $(x_{i_0+1}-x_{i_0})$ are both equal to $\alpha_{i_0+1}$;
if $i \ge i_0+2$ then as $\alpha_i \ge -\eta/6$ we have $\alpha_i / (\alpha_{i_0+1}-\alpha_{i_0}) \ge (-\eta/6) / (1/2) \ge -\eta$;
and if $i \le i_0-1$ then as $\alpha_i + \alpha_{i_0+1}-\alpha_{i_0}\geq-(3\eta)/6=-\eta/2$,
we have $(\alpha_i + \alpha_{i_0+1}-\alpha_{i_0})/(\alpha_{i_0+1}-\alpha_{i_0})\geq -\eta$.
\end{proof}

\subsection{Lower bounds}\label{sec:lb}

In this Section we show that (in the worst-case)
comparison-queries yield no advantage when the bit complexity is large in dimension $d \ge 3$,
or when the dimension is large even if the margin is large.


\subsubsection{Dimension $d\geq3$}
We show that (in the worst-case) comparison-queries do not yield a significant saving in query complexity for learning half spaces,
already in $\mathbb{R}^3$.
This is tight since, as discussed in the introduction,
in $\mathbb{R}^2$ comparison-queries yield an exponential saving.

\begin{theorem}\label{thm:lbdim}
Consider the class $(\R^3,H_3)$ of half spaces in $\R^3$,
Any algorithm that reveals the labels of any realizable sample
of size $n$ must use $\Omega(n)$ comparison/label queries in the worst-case.
\end{theorem}

In the statistical setting, we get that
\begin{corollary}\label{cor:lbdim}
Let $\eps > 0$. Then any algorithm that learns $(\R^3,H_3)$
with error $\eps$ and confidence at least $5/6$
must use $\Omega(1/\eps)$ comparison/label queries
on some realizable distributions.
\end{corollary}

We derive these statements by showing that
the inference dimension of $(\R^3,H_3)$ is $\infty$.
Then, Theorem~\ref{thm:lbdim} and Corollary~\ref{cor:lbdim}
follow by plugging $t=2$ in Theorem~\ref{thm:infdim} and Corollary~\ref{cor:statlb} respectively.
(Note that comparison queries are $2$-local, and thus $t=2$).

\begin{theorem}\label{thm:infdimR3}
The inference dimension of $(\R^3,H_3)$ is $\infty$.
\end{theorem}
\begin{proof}
We need to show that for every $n$, there is $c=\sign(f)\in H_3$,
and a set $X_n\subseteq\R^3$ of $n$ points such that for every
$x\in X_n$, it is {\bf not} the case that $X_n\setminus\{x\}\yields{f} x$.
We use the following lemma.
\begin{lemma}\label{lem:lbdim}
There exists a 3 dimensional linear space $V_n$ of functions $g:\{1,\ldots n\}\to\R$,
and $n+1$ functions $g_0,g_1,\ldots,g_n\in V_n$ with the following properties:
\begin{itemize}
\item  $g_i(j) > 0$ if and only if $i\neq j$ (in particular, $g_0(j) > 0$ for all $j$).
\item $\lvert g_i(1) \rvert < \lvert g_i(2) \rvert < \ldots < \lvert g_i(n) \rvert$ for all $0\leq i\leq n$.
\end{itemize}
\end{lemma}
We first use Lemma~\ref{lem:lbdim} to prove Theorem~\ref{thm:infdimR3}.
Let $v_1,v_2,v_3\in V_{n}$ be a basis for the linear space from Lemma~\ref{lem:lbdim}.
Define
\[x_j \eqdef \bigl(v_1(j),v_2(j),v_3(j)\bigr)\in\R^3\]
and set $X_n=\{x_1,\ldots,x_n\}$.
Note that each function $g_i$ can be represented by a linear combination
$\alpha_{i,1} v_1 + \alpha_{i,2} v_2 + \alpha_{i,3} v_3$.
Define $f_i:\R^3\to \R$ by $f_i=\ip{\alpha_i,x}$, and set $c_i=\sign(f_i)$.
Thus, for all $i,j$,
\begin{equation}\label{eq:indist}
f_i(x_j) = g_i(j).
\end{equation}
Consider the set $X_n\setminus\{x_i\}$.
Equation~\eqref{eq:indist} implies that there is no comparison nor label query that distinguishes
$c_0=\sign(f_0)$ from $c_i=\sign(f_i)$ on this set.
This means that it is {\bf not} the case that $X_n\setminus\{x_i\}\yields{f_0} x_i$,
which finishes the proof of Theorem~\ref{thm:infdimR3}.

It remains to prove Lemma~\ref{lem:lbdim}.
\begin{proof}[Proof of Lemma~\ref{lem:lbdim}]
For a sufficiently large $M$, let $V_n$ denote
the space of functions $g:\{1,\ldots,n\}\to\R$
that is spanned by the three functions
$M^x,x\cdot M^x, x^2\cdot M^x.$
Define
\[g_i(x) = M^x\bigl(1 - 2(x-i)^2\bigr)\in V_n.\]
It is easy to check that the first item in the conclusion is satisfied
by the $g_i$'s, and that if $M$ is sufficiently large, then so does the second item.
\end{proof}

\end{proof}

\subsubsection{Margin}

We show here that, in the worst-case,
comparison-queries do not yield a significant
saving in query complexity for learning half spaces,
even if it is guaranteed that the margin is large, say at least $\nicefrac{1}{8}$.

\begin{theorem}\label{thm:lbmarg}
For every $n$ there is a class $(X,H)$, where $X\subseteq \R^{n+1}$,
and $H\subseteq H_{n+1}$ contains all the half spaces with margin at least $\nicefrac{1}{8}$
such that the following holds:
any algorithm that reveals the labels of any realizable sample
of size $n$ must use $\Omega(n)$ comparison/label queries in the worst-case.
\end{theorem}

In the statistical setting, we get that
\begin{corollary}\label{cor:lbmarg}
For every $\eps > 0$, there is $n$ and a class $(X,H)$, where $X\subseteq \R^{n+1}$,
and $H\subseteq H_{n+1}$ contains all the half spaces with margin at least $\nicefrac{1}{8}$
such that the following holds:
any algorithm that learns $(X,H)$
with error $\eps$ and confidence at least $\nicefrac{5}{6}$
must use $\Omega(\nicefrac{1}{\eps})$ comparison/label queries
on some realizable distributions.
\end{corollary}

We derive these statements by establishing the existence
of classes with large margin and large infrence dimension.
Then, Theorem~\ref{thm:lbmarg} and Corollary~\ref{cor:lbmarg}
follow by plugging $t=2$ in Theorem~\ref{thm:infdim} and Corollary~\ref{cor:statlb} respectively.

\begin{theorem}\label{thm:lbmargin}
For every $n$, there is a set of $n$ unit vectors $X=\{x_1,x_2,\ldots,x_n\}\subset \R^{n+1}$
such that the class $(X,H)$ has inference dimension at least $n$, where $H$ contains
all half spaces with margin at least $\nicefrac{1}{6}$ with respect to $X$.
\end{theorem}
\begin{proof}
Define $x_i \eqdef \frac{e_i + e_{n+1}}{\sqrt{2}}$ where the $e_i$'s are the vectors in the standard basis,
and define $w_i\in\R^{n+1}$ for $i=1,\ldots,n$ by
\[w_i(j) \eqdef
\begin{cases}
1 + \frac{j}{10n^2} &j =i\\
-\frac{1}{2} &j =n+1\\
-\frac{j}{10n^2} &\text{otherwise}
\end{cases}
 \]
 and $w_0\in\R^{n+1}$ by
 \[w_0(j) \eqdef
\begin{cases}
-\frac{1}{2} &j =n+1\\
-\frac{j}{10n^2} &\text{otherwise}
\end{cases}
 \]
Let $f_i(x) = \ip{w_i,x}$, and define $c_i=\sign(f_i)$, for $i=0,\ldots,n$.
A short calculation shows that the margin $\gamma(c_i) = \min_j\frac{\ip{w_i,x_j}}{\lvert\lvert w_i\rvert\rvert_2}\geq\nicefrac{1}{8}$, and therefore, $c_i\in H$ for all $0\leq i\leq n$.
Another calculation shows that $c_i(x_j) =+1 $ if and only if $i=j$,
and that $\bigl\lvert f_i(x_j)\bigr\rvert = \frac{\nicefrac{1}{2} + \nicefrac{j}{10n^2}}{\sqrt{2}}$.
In particular note that $\lvert f_i(x_j)\rvert$ does not depend on $i$.

Consider the set $X_n\setminus\{x_i\}$.
The previous paragraph implies that there is no comparison nor label query that distinguishes
$c_0=\sign(f_0)$ from $c_i=\sign(f_i)$ on this set.
This means that it is {\bf not} the case that $X_n\setminus\{x_i\}\yields{f_0} x_i$,
which finishes the proof of Theorem~\ref{thm:lbmargin}.
\end{proof}

\section{Acknowledgements}
This work has benefited from various
discussions during the special program
on \emph{Foundations of Machine Learning} that took place at the
\emph{Simons Institute for the Theory of Computing}, in Berkeley.
In particular, the authors would like to thank
Sanjoy Dasgupta for inspiring the 
focus on comparison queries.

\bibliographystyle{abbrvnat}
\bibliography{paper}


\end{document}